%% file: root.tex
\documentclass[letterpaper, 10 pt, conference]{ieeeconf}

\IEEEoverridecommandlockouts                              
                                                          
\input{mytools}

\input{mygloss}

\makeatletter
\let\NAT@parse\undefined
\makeatother
\usepackage[colorlinks=true,
    citecolor=blue,
    menucolor=blue,
    linkcolor=blue,
    urlcolor=magenta	
]{hyperref}

\title{\LARGE \bf
Towards safe Bayesian optimization with Wiener kernel regression
}

\author{Oleksii Molodchyk$^\star$, Johannes Teutsch$^\star$, and Timm Faulwasser$^\diamond$
\thanks{$^\star$: Equally contributing first authors. $^\diamond$: Corresponding author.
OM and TF acknowledge funding in the course of TRR 391 \textit{Spatio-temporal Statistics for the Transition of Energy and Transport} (520388526) by the Deutsche Forschungsgemeinschaft (DFG,
German Research Foundation).
}
\thanks{OM and TF are with the Institute of Control Systems,
        Hamburg University of Technology, 21073 Hamburg, Germany
        {\tt\small oleksii.molodchyk@tuhh.de} and {\tt\small timm.faulwasser@ieee.org}.}
\thanks{JT is with the Chair of Automatic Control Engineering,
        Technical University of Munich, 80333 Munich, Germany
        {\tt\small johannes.teutsch@tum.de}.}
}

\begin{document}

\maketitle
\thispagestyle{empty}
\pagestyle{empty}

\begin{abstract}
        Bayesian Optimization (BO) is a data-driven strategy for minimizing/maximizing black-box functions based on probabilistic surrogate models. In the presence of safety constraints, the performance of BO crucially relies on tight probabilistic error bounds related to the uncertainty surrounding the surrogate model. For the case of Gaussian Process surrogates and Gaussian measurement noise, we present a novel error bound based on the recently proposed Wiener kernel regression. We prove that under rather mild assumptions, the proposed error bound is tighter than bounds previously documented in the literature, leading to enlarged safety regions. We draw upon a numerical example to demonstrate the efficacy of the proposed error bound in safe BO.
\end{abstract}

\input{1_intro}
\input{2_prelim}

\input{3_safeBO}
\input{4_main}
\input{5_eval}
\input{6_conclusion}

\input{appendix}

\addtolength{\textheight}{-15.3cm} 

\section*{ACKNOWLEDGMENT}
The authors thank Jannis Lübsen and Johannes Köhler for valuable discussions and feedback.  

\bibliographystyle{IEEEtran}
\bibliography{literature}

\end{document}

%% file: mytools.tex
\usepackage[final]{microtype}
\usepackage{graphicx} 
\usepackage{epsfig} 
\usepackage{times} 
\usepackage{amsmath} 
\usepackage{amssymb} 
\usepackage{amsfonts}
\usepackage{mathrsfs}
\usepackage{algorithm}
\usepackage{algcompatible}
\usepackage{url}

\newcommand{\pcecoe}[2]{\mathsf{#1}^{#2}}

\newcommand{\mean}[1]{\mathbb{E}\left[#1\right]}
\newcommand{\var}[1]{\mathbb{V}\left[#1\right]}

\newcommand{\mat}[1]{\ensuremath{\begin{bmatrix} #1 \end{bmatrix}}}

\renewcommand{\rho}{\varrho}
\renewcommand{\exp}[1]{\mathrm{exp}\left(#1\right)}
\renewcommand{\ln}[1]{\mathrm{ln}\left(#1\right)}
\renewcommand{\det}[1]{\mathrm{det}\left(#1\right)}
\newcommand{\norm}[1]{\left\lVert #1 \right\rVert}
\newcommand{\abs}[1]{\left\lvert #1 \right\rvert}
\newcommand{\inprod}[1]{\left\langle #1 \right\rangle}

\newtheorem{theorem}{Theorem}
\newtheorem{lemma}{Lemma}
\newtheorem{definition}{Definition}
\newtheorem{assumption}{Assumption}
\newtheorem{remark}{Remark}

%% file: mygloss.tex
\usepackage[acronym, style=alttree, toc=true, shortcuts, xindy, nomain, nonumberlist]{glossaries}

\newglossary{symbols}{sym}{sbl}{List of Symbols}

\newglossary{notation}{not}{nt}{Notation}

\glsaddkey{dimension}{\glsentrytext{\glslabel}}{\glsentryd}{\GLsentryd}{\glsd}{\Glsd}{\GLSd}

\newacronym{bo}{BO}{Bayesian Optimization}
\newacronym{gp}{GP}{Gaussian Process}
\newacronym{pce}{PCE}{Polynomial Chaos Expansion}
\newacronym{rkhs}{RKHS}{Reproducing Kernel Hilbert Space}
\newacronym{iid}{i.i.d.}{independent and identically distributed}

\newglossaryentry{kermean}{type=symbols,
	sort={mu},
	dimension={\ensuremath{ n_y }},
	name={\ensuremath{\mu}},
	description={Mean predictor from kernel/GP regression}
}

\newglossaryentry{kerstd}{type=symbols,
	sort={sigmaGP},
	dimension={\ensuremath{ 1 }},
	name={\ensuremath{\sigma_{\mathrm{GP}}}},
	description={GP posterior standard deviation}
}

\newglossaryentry{wiestd}{type=symbols,
	sort={sigmaWK},
	dimension={\ensuremath{ 1 }},
	name={\ensuremath{\sigma_{\mathrm{WK}}}},
	description={Wiener kernel standard deviation}
}

\newglossaryentry{gram}{type=symbols,
	sort={K},
	dimension={\ensuremath{ D \times D }},
	name={\ensuremath{\mathbf{K}}},
	description={Gram matrix}
}

\newglossaryentry{gramvec}{type=symbols,
	sort={k},
	dimension={\ensuremath{ D }},
	name={\ensuremath{\mathbf{k}}},
	description={Data-dependent feature vector based on kernel}
}

\newglossaryentry{bound}{type=symbols,
	sort={bound},
	dimension={\ensuremath{ 1 }},
	name={\ensuremath{\eta}},
	description={Bounding function for probabilistic uniform error bound}
}

\newglossaryentry{xsafe}{type=symbols,
	sort={xsafe},
	dimension={\ensuremath{ \mathcal{X}_{\mathrm{safe}} }},
	name={\ensuremath{x_{\mathrm{safe}}}},
	description={Guaranteed safe action}
}

\glsdisablehyper

%% file: 1_intro.tex
\section{INTRODUCTION}
Optimizing black-box functions with noisy and costly evaluations is a common challenge in science and engineering, e.g., consider tuning controller parameters in closed-loop operation.  \gls{bo} is a well-suited machine learning method for this task, particularly when evaluations are noisy, expensive, and time-consuming~\cite{shahriari2015taking}. Unlike traditional gradient-based optimization methods, \gls{bo} builds a probabilistic surrogate model of the objective function through data, often using a \gls{gp} \cite{Rasmussen2006}. The \gls{gp} surrogate model yields estimates of both---the function output via the mean predictor as well as the uncertainty surrounding the prediction via the \gls{gp} posterior variance. By balancing exploitation of knowledge acquired from observations and exploration of yet unknown regions of the hypothesis space, \gls{bo} can effectively find the optimum of the objective with few function evaluations. 

In many applications, only control actions or parameters shall be considered that guarantee \textit{safe} operation of the closed-loop system, specified via specific constraints. Safe \gls{bo} algorithms have been developed to address this problem, e.g., the seminal SafeOpt framework \cite{sui2015safe}. Safe \gls{bo} guarantees satisfaction of safety constraints with high probability by employing statistical bounds for the error between the surrogate model and the unknown function \cite{Srinivas2009,Chowdhury2017}. Safe \gls{bo}  finds application in automatic controller tuning \cite{berkenkamp2016safe,lubsen2024towards}, safe robot learning \cite{baumann2021gosafe}, and real-time/feedback optimization of chemical reactors \cite{Krishnamoorthy2023}, among others. We refer to \cite[Sec.~4]{fiedler2024safety} for an overview of the recent literature on safe \gls{bo}.

Crucially, the performance of safe \gls{bo} algorithms depends on the tightness of the probabilistic error bound: conservative bounds restrain the exploration of the hypothesis space, leading to slow convergence and inefficient optimization. Indeed, there are two sources of uncertainty surrounding the surrogate model -- insufficient exploration of the hypothesis space and noise corruption of observed data. As the \gls{gp} posterior variance (including the additive noise variance) considers both sources of uncertainty, commonly used approaches express the error bound in terms of this variance \cite{Srinivas2009,Chowdhury2017,abbasi2013online,Fiedler2021}. However, our recent work on \textit{Wiener kernel regression}  has proposed a framework towards untangling the noise-induced uncertainty in the data from the \gls{gp} posterior variance~\cite{faulwasser2024}. Therein, we explicitly consider measurement noise in the regression via \gls{pce} of the random variables \cite{Xiu2002}. The \gls{pce} approach yields an expression for the output variance induced by the noise in the data. We remark that \gls{pce} dates back to Norbert Wiener~\cite{Wiener1938}, it is commonly used in uncertainty quantification~\cite{Sullivan2015}, and it has seen use in systems and control, see, e.g.,~\cite{tudo:faulwasser23a}.  

In this work, we propose a novel probabilistic error bound for \gls{gp}/kernel predictors by leveraging insights from Wiener kernel regression, focusing on the case of (sub-) Gaussian measurement noise. Under mild assumptions, we show that the proposed bound is tighter than commonly used bounds documented in the literature \cite{abbasi2013online,Fiedler2021}, and thus allows for enlarged safe regions in \gls{bo}. Drawing upon a numerical example, we demonstrate that safe \gls{bo} using the proposed bound outperforms \gls{bo} based on existing bounds in terms of cumulative regret and size of the safe region.

The remainder of the paper is structured as follows: In Section~\ref{sec:prelim}, we revisit \gls{gp} and Wiener kernel regression. The considered safe \gls{bo} setup is introduced in Section~\ref{sec:safeBO}. Section~\ref{sec:main} presents the proposed uncertainty bound and a qualitative comparison to bounds from literature. In Section~\ref{sec:eval}, we evaluate safe \gls{bo} based on the proposed error bound, before we draw conclusions in Section~\ref{sec:conclusion}.

\subsubsection*{Notation} Let $(\Omega,\, \mathcal{F},\, \mathbb{P})$ be a probability space with the set of outcomes $\Omega$, $\sigma$-algebra $\mathcal{F}$, and probability measure $\mathbb{P}$ \cite{Sullivan2015}. A scalar random variable is a measurable function $V: \Omega \to \mathbb{R}$ and its realization for an outcome $\omega \in \Omega$ is $V(\omega) \in \mathbb{R}$. We write $V \in (\Omega,\, \mathcal{F},\, \mathbb{P}; \mathbb R)$. If $V$ has finite expectation $\mean{V}$ and finite variance $\var{V}$, we write $V \in L^2(\Omega,\, \mathcal{F},\, \mathbb{P};\, \mathbb{R}^{n_v})$. We call $\nu \doteq \mathbb{P} \circ V^{-1}$ the distribution of $V$, and we write $V \sim \nu$. A Gaussian distribution with mean $\mu$ and (co-)variance $\sigma^2$ is denoted via $\mathcal{N}(\mu, \sigma^2)$. We use $\mathbf{I}$ for the identity matrix and $\mathbb{I}_D \doteq [1, D] \cap \mathbb{N}$ for $D \in \mathbb{N}$.

%% file: 2_prelim.tex
\section{PRELIMINARIES} \label{sec:prelim}

\subsection{Gaussian Processes and Reproducing Kernels}
Consider an unknown function $f: \mathcal{X} \to \mathcal{Y}$ and a dataset
\begin{equation} \label{eq:dataset}
        \mathcal{D} = \left\{(x_i,\, y_i) \in \mathcal{X} \times \mathcal{Y} \mid i \in \mathbb{I}_D \right\}
\end{equation}
obtained from finitely many input evaluations for $\mathbf{x} = (x_1, \dots,  x_D)$. Even though GPs can be derived for vector-valued functions, for the sake of simplicity, we limit ourselves to scalar-valued functions, i.e., $\mathcal{Y} = \mathbb{R}$. We have the following standard GP assumption.
\begin{assumption}[Gaussian i.i.d. noise] \label{assum:gauss_noise}
    The input values $\mathbf{x}$ can be chosen freely in $\mathcal{X} \neq \emptyset$ whereas the corresponding labels $\mathbf{y} = \mat{y_1 & \dots & y_D}^\top \in \mathbb{R}^D$ are generated via
    \begin{equation} \label{eq:system}
        y_i = f(x_i) + M_i(\omega), \quad M_i \sim \mathcal{N}(0,\, \sigma_M^2), \quad i \in \mathbb{I}_D,
    \end{equation}
    where $M_i(\omega)$ are realizations of \gls{iid} Gaussian random variables $M_i$.
    \hfill {\small{$\Box$}}
\end{assumption}

Considering the model \eqref{eq:system} and following the weight-space view on \gls{gp}s \cite{Rasmussen2006}, we assume that the unknown function admits a decomposition $f(x) = \mathrm{w}^\top \phi(x)$ in terms of known, possibly nonlinear features $\phi: \mathbb{R}^{n_x} \to \mathbb{R}^{n_\phi}$ and unknown weights $\mathrm{w} \in \mathbb{R}^{n_\phi}$. We consider a zero-mean Gaussian prior distribution over the weights, i.e., $W \sim \mathcal{N}(0,\, \sigma_W^2 \mathbf{I})$, and we leverage the Gaussian likelihood of the labels $\mathbf{y}$ due to \eqref{eq:system}. Applying Bayes' Theorem to the weights and calculating the inner product of the resulting posterior with the features $\phi(x)$ yields the well-known mean and posterior variance estimate
\begin{subequations} \label{eq:gp}
    \begin{align} 
        \gls{kermean}(x) &= \gls{gramvec}(x)^{\top}\left(\gls{gram}+\sigma_M^2\mathbf{I}\right)^{-1}\mathbf{y}, \label{eq:kermean} \\
        \gls{kerstd}^2(x) &= k(x,x) - \gls{gramvec}(x)^{\top}\left(\gls{gram}+\sigma_M^2\mathbf{I}\right)^{-1} \gls{gramvec}(x). \label{eq:kervar}
    \end{align}
\end{subequations}
Here, $k(x,x) = \sigma_W^2 \phi(x)^\top \phi(x)$ is the kernel function, $\gls{gramvec}(x) \doteq \begin{bmatrix}
    k(x, x_1) & \ldots & k(x, x_D)
\end{bmatrix}^\top$ is the vector of kernels centered at input values $\mathbf{x}$, whereas $\mathbf{K} \doteq \left[\kappa(x_i,x_j)\right]_{i\in \mathbb{I}_D,j \in \mathbb{I}_D}$ is the Gramian constructed from the $\mathbf{x}$-data.

The kernel function is the centerpiece of any GP as it defines the prior covariance between the values of a black-box function at any two points $x, x^\prime \in \mathcal{X}$, and thus, the hypothesis space. It also defines a \gls{rkhs}. 

\begin{definition}[Kernel function] \label{def:kernel_function}
    A function $k: \mathcal{X} \times \mathcal{X} \to \mathbb{R}$ is called a kernel function (or a \emph{reproducing kernel}) if for any finite collection $\mathbf{x} \in \mathcal{X}^D$ of input values, the induced Gram matrix $\mathbf{K}$ is symmetric and positive semi-definite. \hfill {\small{$\Box$}}
\end{definition}

\begin{definition}[\gls{rkhs} \cite{Aronszajn1950, Berlinet2004}]\label{def:rkhs}
    The space $\mathcal{H}_k$ of functions $\mathcal{X} \to \mathbb{R}$ is a \gls{rkhs} if it is a Hilbert space equipped with a real-valued inner product $\langle\cdot,\cdot\rangle_{\mathcal{H}_k}$, and if there exists  $k: \mathcal{X} \times \mathcal{X} \to \mathbb{R}$ such that:
    \begin{itemize}
        \item[i)] For any $x \in \mathcal{X}$, it holds that $k(\cdot, x) \in \mathcal{H}_k$.
        \item[ii)] For any $x \in \mathcal{X}$ and any function $f \in \mathcal{H}_k$, the value $f(x)$ can be reproduced via $f(x) = \langle f, k(\cdot, x)\rangle_{\mathcal{H}_k}$. \hfill {\small{$\Box$}}
    \end{itemize}  
\end{definition}
Note that if a function $k$ in Definition~\ref{def:rkhs} exists, then it is a kernel function, cf.~Definition~\ref{def:kernel_function}. Furthermore, due to the Moore-Aronszajn Theorem \cite{Aronszajn1950}, for any Hilbert space satisfying the properties i) and ii) above, the underlying $k$ is unique. Hence, the notation $\mathcal{H}_k$.

Associated with the inner product on $\mathcal{H}_k$ is the induced \gls{rkhs} norm $\norm{\cdot}_{\mathcal{H}_k} \doteq \sqrt{\langle \cdot, \cdot \rangle_{\mathcal{H}_k}}$. This norm plays an important role in the derivation of error bounds on the mean estimate \eqref{eq:kermean}. For example, for noise-free data \eqref{eq:dataset}, there exists the following well-known result.
\begin{lemma}[Error bound with noise-free data {\cite[p.~38]{Fasshauer2011}}]
\label{lem:bound_noise_free}
    Assume that the unknown function $f: \mathcal{X} \to \mathbb{R}$ belongs to the \gls{rkhs} $\mathcal{H}_k$ and that the \gls{gp} \eqref{eq:gp} is trained using noise-free data $\mathcal{D}$, i.e., with labels obtained via \eqref{eq:system} with $\sigma_M = 0$. Then,
    \begin{equation*}
        \abs{f(x) - \mu(x)} \leq \norm{f}_{\mathcal{H}_k} \sigma_\mathrm{GP}(x), \quad \forall \, x \in \mathcal{X},
    \end{equation*}
    with $\gls{kermean}(x)$ and $\gls{kerstd}(x)$ from \eqref{eq:gp} with $\sigma_M^2 = 0$. \hfill {\small{$\Box$}}
\end{lemma}
Since $f$ is unknown, one usually replaces its \gls{rkhs} norm with some upper-bound estimate $B \geq \norm{f}_{\mathcal{H}_k}$.

\subsection{Wiener Kernel Regression}
The theory of \gls{rkhs} is the foundation for kernel regression, where an objective functional representing the goodness-of-fit to data $\mathcal{D}$ is minimized over the \gls{rkhs} associated with some user-specified kernel function. For example, the problem
\begin{equation} \label{eq:kernel_reg}
    \min_{\mathrm{w} \in \mathbb{R}^{n_\phi}}  \sum_{i=1}^D \left( y_i - \mathrm{w}^\top \phi(x_i)\right)^2 +\varrho^2 \norm{\mathrm{w}}^2
\end{equation}
is a kernel regression in the \gls{rkhs} associated with $k(x,x^\prime) = \phi(x)^\top \phi(x^\prime)$ which is due to the kernel trick \cite{Schoelkopf2001}. The above problem is connected to \gls{gp} regression since its solution corresponds to the mean estimate \eqref{eq:kermean} for $\varrho = \sigma_M / \sigma_W$, where without loss of generality, we consider the prior variance of the coefficients to be $\sigma_W^2 = 1$, and we set $\varrho = \sigma_M$. 

The core idea behind Wiener kernel regression~\cite{faulwasser2024} is to exploit the knowledge of the noise distribution. This leads to an improved estimate of \eqref{eq:kernel_reg} as in the regression, the noise model is subtracted from the labels $y_i$. However, since the noise realizations $M_i(\omega)$ are unknown, one can replace them only with their probabilistic representations $M_i = \sigma_M \xi_i$ for all $i \in \mathbb{I}_D$, wherein $\xi_i \sim \mathcal{N}(0,1)$ are \gls{iid} This modification turns the objective of \eqref{eq:kernel_reg} into a random variable living on the probability space $L^2(\Omega, \mathcal{F}, \mathbb{P}; \mathbb R)$ of random variables with finite expectation and variance. Minimizing the objective from~\eqref{eq:kernel_reg} in the expected value sense gives
\begin{equation*} 
    \min_{W \in L^2(\cdot)
    } \sum_{i=1}^D \mean{\left( y_i - \sigma_M \xi_i - W^\top \phi(x_i)\right)^2 +\sigma_M^2 \norm{W}^2},
\end{equation*}
where $ \norm{W}^2 = W^\top W$.
Notice that in the minimization and due to the linearity of the first-order optimality condition, the weights $W$ are lifted to $L^2(\Omega, \mathcal{F}, \mathbb{P}; \mathbb R^{D})$ with the same distribution as the vectorized noise $M_1,\, \dots,\, M_D$. 

The main insight of our previous work \cite{faulwasser2024} is that a minimizer $W^\star \in L^2(\Omega, \mathcal{F}, \mathbb{P}; \mathbb R^{D})$ of the above problem can be expressed as the series
$W^\star = \pcecoe{w}{0 \star} + \sum_{i=1}^D \pcecoe{w}{i \star} \xi_i$
with coefficients $\pcecoe{w}{j \star} \in \mathbb{R}^{n_\phi}$, $\forall \, j \in \{0\} \cup \mathbb{I}_D$. Furthermore, these deterministic coefficients can be individually (and uniquely) obtained by decomposing the above problem and using the kernel trick on each one of the $D+1$ sub-problems~\cite[Lem.~1]{faulwasser2024}.

The resulting point-estimate of the black-box function for a given $x \in \mathcal{X}$ can then be expressed as a random variable $W^{\star\top} \phi(x)$ with the mean identical to \eqref{eq:kermean} and variance 
\begin{equation} \label{eq:wievar}
    \gls{wiestd}^2(x) = \sigma_M^2\gls{gramvec}(x)^{\top}\left(\gls{gram}+\sigma_M^2\mathbf{I}\right)^{-2} \gls{gramvec}(x).
\end{equation}
Note that in terms of computational complexity, Wiener kernel regression is identical to GPs as it is centered around inverting a $D \times D$ matrix.

In \cite{faulwasser2024} 
we use \glspl{pce} to consider arbitrary non-Gaussian measurement noise in $L^2(\Omega, \mathcal{F}, \mathbb{P})$. Henceforth, however, we focus only on the Gaussian setting.

%% file: 3_safeBO.tex
\section{PROBLEM STATEMENT: SAFE BO} \label{sec:safeBO}

\gls{bo} aims to find the optimum of a black-box function $g: \mathcal{X} \to \mathcal{Y}$ by sequentially evaluating actions $x_t$ and observing their outcomes $y_{g,t}$ via \eqref{eq:system}. It systematically builds a probabilistic surrogate model for the unknown function $f$ (e.g., via \gls{gp} regression, see \eqref{eq:gp}). This model is then used to construct an acquisition function $\alpha^g: \mathcal{X} \to \mathbb{R}$, indicating how beneficial evaluating a particular action $x_t$ might be in relation to the uncertainty surrounding the true function $g$.  Optimizing the acquisition function, balancing exploration and exploitation of $g$, gives the next action to be applied. With each iteration, the surrogate model is updated based on new observations $(x_t,\,y_{g,t})$, constituting the learning phase.

In the presence of safety constraints, the \gls{bo} framework must be adjusted to account for safe exploration \cite{sui2015safe}. Similarly to the black-box objective function $g$, the constraints are often replaced by probabilistic surrogate models stemming from \gls{gp} regression. Without loss of generality, let us consider one safety constraint of the form $f(x_t) \le 0$ with unknown constraint function $f: \mathcal{X} \to \mathbb{R}$ from which we can obtain a noisy observation $y_{f,t}$ after applying the action $x_t$ via \eqref{eq:system}.\footnote{For the general case of multiple safety constraints, we refer to \cite{Krishnamoorthy2023}.} The safety constraint defines an (apriori unknown) region $\glsd{xsafe}(f) \doteq \{x\in\mathcal{X} \mid f(x) \le 0 \}$ of safe actions $x_t$. For safe \gls{bo}, we rely on the following assumption common in safe learning \cite[Assump.~1]{Krishnamoorthy2023}.
\begin{assumption}[Safe action] \label{assum:safe}
    It holds that $\glsd{xsafe}(f) \neq \emptyset$, and a safe action $\gls{xsafe} \in \glsd{xsafe}(f)$ is known. \hfill {\small{$\Box$}}
\end{assumption}

The action $\gls{xsafe}$ is used as a safe (but suboptimal) backup if no feasible solution to the \gls{bo} problem can be found, e.g., due to large uncertainty in the surrogate model.
In order to guarantee safety with high probability, probabilistic error bounds for the surrogate model are of utmost importance. Consider the following standard definition of probabilistic error bounds for the regression error (cf.~\cite[Def.~2.1]{Lederer2019}).
\begin{definition}[Probabilistic error bound] \label{def:bound}
    Let ${f: \mathcal{X} \to \mathcal{Y}}$ be an unknown function following the measurement equation~\eqref{eq:system} and let $\gls{kermean}: \mathcal{X} \to \mathcal{Y}$ be the mean estimate \eqref{eq:kermean} based on the noisy dataset \eqref{eq:dataset}. The regression exhibits a probabilistically bounded error on a compact set $\mathcal{X} \subset \mathbb{R}^{n_x}$ if there exists a function $\gls{bound}: \mathcal{X} \to \mathbb{R}$ such that for all $x \in \mathcal{X}$
    \begin{equation} \label{eq:errorbound_def}
       \mathbb{P}\left[\abs{ f(x) - \gls{kermean}(x) } \le \gls{bound}(x)\right] \ge 1-\delta
    \end{equation}
    holds for some $\delta \in (0,\,1)$. \hfill \small{$\Box$}
\end{definition}

With an error bound as the one from Definition~\ref{def:bound}, one can define the upper confidence bound (UCB) $\alpha^f: \mathcal{X} \to \mathbb{R}$,
\begin{equation} \label{eq:ucb}
    \alpha^f(x_t) \doteq \gls{kermean}(x_t) + \gls{bound}(x_t),
\end{equation}
as the surrogate model for the safety constraint function $f(x_t)$.
In fact, UCBs of the form \eqref{eq:ucb} are also commonly used for the acquisition function in \gls{bo}, i.e., as a surrogate model $\alpha^g$ for the to-be-optimized function $g$. Such acquisition functions implicitly consider a trade-off between exploitation and exploration via the mean predictor $\gls{kermean}(x_t)$ and the confidence bound $\gls{bound}(x_t)$. We focus on acquisition functions of the form \eqref{eq:ucb} in this work. For a detailed discussion on alternative acquisition functions, we refer the interested reader to \cite{garnett2023bayesian}.

The corresponding safe \gls{bo} problem that is iteratively solved (e.g., by following an interior-point approach~\cite{Krishnamoorthy2023}) to optimally choose safe actions $x_t$ at time $t$ is formulated as
\begin{align} \label{eq:safeBO}
    x_t = &~ \arg\underset{\tilde{x} \in \mathcal{X}}{\max}~~ \alpha^g(\tilde{x}) ~~\mathrm{s.t.}~~\alpha^f(\tilde{x})\le 0.
\end{align}
If the partially revealed safe region is empty, i.e., $\glsd{xsafe}(\alpha^f) = \{x\in\mathcal{X} \mid \alpha^f(x) \le 0 \} = \emptyset$, problem~\eqref{eq:safeBO} is infeasible. In that case, we choose the safe action $\gls{xsafe}$ relying on Assumption~\ref{assum:safe}. Algorithm~\ref{alg:safeBO} summarizes the safe \gls{bo} scheme.

\begin{algorithm}[t]
\caption{Safe \gls{bo}} \label{alg:safeBO}
\begin{algorithmic}[1]
\REQUIRE Surrogate models $\alpha^f$, $\alpha^g$ \eqref{eq:ucb}, end time $T$ 
\FOR{$t \in \mathbb{I}_{T}$}
\IF{\eqref{eq:safeBO} is feasible} 
\State Choose action $x_t$ as maximizer of \eqref{eq:safeBO}
\ELSE 
\STATE Choose action $x_t = \gls{xsafe}$
\ENDIF
\STATE Observe $y_{f,t}$, $y_{g,t}$ and update surrogate models $\alpha^f$, $\alpha^g$
\ENDFOR
\end{algorithmic}
\end{algorithm}

\begin{table*}[!t]
    \renewcommand{\arraystretch}{2.2}
    \setlength{\tabcolsep}{12pt}
    \centering
    \caption{Overview of probabilistic uniform error bounds $\gls{bound}(x)$ (see Definition~\ref{def:bound}) for Gaussian measurement noise}
    \begin{tabular}{lll} \hline
        \cite[Thm.~3.11]{abbasi2013online}: & $\gls{bound}_1(x) = \left(B + \beta_1(\delta)\right) \gls{kerstd}(x)$ & $\beta_1(\delta) \doteq \sqrt{\ln{\det{\sigma_M^{-2}\gls{gram} + \mathbf{I}_D}} + 2\ln{1/\delta}}$\\
        \cite[Prop.~2]{Fiedler2021}: & $\gls{bound}_2(x) = B \gls{kerstd}(x) + \beta_2(\delta) \sigma_M  \norm{ \left(\gls{gram} + \sigma_M^2 \mathbf{I}_D\right)^{-1}\gls{gramvec}(x)}$~~~ & $\beta_2(\delta) \doteq \sqrt{D + 2\sqrt{D}\sqrt{\ln{1/\delta}} + 2\ln{1/\delta}}$\\ 
        Theorem~1: & $\gls{bound}_{\mathrm{WK}}(x) = B \sqrt{\gls{kerstd}^2(x) - \gls{wiestd}^2(x)} + \beta_{\mathrm{WK}}(\delta)\gls{wiestd}(x)$ & $\beta_{\mathrm{WK}} \doteq \sqrt{2\ln{2/\delta}}$ \\ \hline
    \end{tabular}
    \label{tab:bounds}
    \renewcommand{\arraystretch}{1}
\end{table*}

The following result extends \cite[Thm.~1]{Krishnamoorthy2023} to general confidence bounds of the form \eqref{eq:errorbound_def}.
\begin{lemma}[Safety constraint satisfaction] \label{lem:safety}
    Given Assumption~\ref{assum:safe} and a probabilistic error bound as in  \eqref{eq:errorbound_def}, the safety constraint $f(x_t)\le 0$ is satisfied with a probability of at least $1-\delta$ for the action $x_t$ chosen via Algorithm~\ref{alg:safeBO}. \hfill \small{$\Box$} 
\end{lemma}
\begin{proof}
    The confidence bound \eqref{eq:errorbound_def} implies that $f(x_t) \le \alpha^f(x_t)$ holds with a probability of at least $1-\delta$ for all $x_t \in \mathcal{X}$. If $\glsd{xsafe}(\alpha^f) \neq \emptyset$, then problem \eqref{eq:safeBO} is feasible and the log-barrier term in \eqref{eq:safeBO} ensures that the chosen action $x_t$ lies in the interior of the partially revealed safe region, i.e., $x_t \in \glsd{xsafe}(\alpha^f)$. Thus, the safety constraint holds with a probability of at least $1-\delta$. If $\glsd{xsafe}(\alpha^f) = \emptyset$, then the safe action $x_t = \gls{xsafe} \in \glsd{xsafe}(f)$ is chosen, satisfying the safety constraint with probability $1$, i.e., almost surely.
\end{proof}

While Algorithm~\ref{alg:safeBO} guarantees safe \gls{bo} with high probability, its performance crucially depends on the tightness of the error bound $\gls{bound}(x)$: If $\gls{bound}(x)$ is conservative, then the safe region $\glsd{xsafe}(\alpha^f)$ expands slowly, leading to slow convergence to the optimum of the unknown function $g$. Thus, tight error bounds of the form \eqref{eq:errorbound_def} are desirable.

\begin{remark}[Overview of Existing Error Bounds] \label{rem:bounds}
The original SafeOpt algorithm \cite{sui2015safe} employed the error bound from \cite[Thm.~6]{Srinivas2009}, which only holds for bounded measurement noise and is thus not applicable to the Gaussian setting.
For $\varrho > 1$ (which is set to $\varrho = \sigma_M$ in \gls{gp} regression, see Section~\ref{sec:prelim}), an improved bound that allows for sub-Gaussian noise is presented in \cite[Thm.~2]{Chowdhury2017}, hence applicable to both the bounded and the Gaussian setting.
However, both \cite[Thm.~6]{Srinivas2009} and \cite[Thm.~2]{Chowdhury2017} rely on an upper bound on the maximum information gain, which introduces conservatism \cite{Srinivas2009}. Less conservative bounds are presented in \cite{Fiedler2021}, namely \cite[Thm.~1]{Fiedler2021} and \cite[Prop.~2]{Fiedler2021}, allowing for $\varrho > 0$ and sub-Gaussian noise. As pointed out in \cite{fiedler2024safety}, the bound in \cite[Thm.~1]{Fiedler2021} coincides with \cite[Thm.~3.11]{abbasi2013online} for $0 < \varrho \le 1$, and \cite[Thm.~3.11]{abbasi2013online} is less conservative than \cite[Thm.~1]{Fiedler2021} for $\varrho > 1$.
The bounds from \cite[Thm.~3.11]{abbasi2013online} and \cite[Prop.~1]{Fiedler2021} are listed in Table~\ref{tab:bounds} for the case of zero-mean i.i.d. Gaussian noise with variance $\sigma_M^2$.
Note that all previously discussed bounds rely on an upper bound $B$ on the \gls{rkhs} norm of the unknown function in Definition~\ref{def:bound}, i.e., $\norm{f}_{\mathcal{H}_k} \le B$.
Furthermore, it deserves to be noted that the results of  \cite{Srinivas2009, Chowdhury2017} use a Martingale setting to analyze the error bounds and the behavior of the safe BO iterates. They also analyze performance/regret in this setting. In contrast, the result of \cite{Fiedler2021} does not rely on this setting but does not quantify regret bounds. Similar to \cite{Fiedler2021}, we also do not use the Martingale setting. The extension of our analysis towards this and towards regret bounds is subject to future work.
\qquad \qquad$\square$

\end{remark}

%% file: 4_main.tex
\section{MAIN RESULTS} \label{sec:main}
In this section, we first derive a novel bound of the form \eqref{eq:errorbound_def} based on Wiener kernel regression for the case of Gaussian measurement noise. Then, we prove that the proposed bound is tighter than bounds known in the literature, cf.~Table~\ref{tab:bounds}. We first present a technical result extending Lemma~\ref{lem:bound_noise_free}.
\begin{lemma}[Extension of Lemma~\ref{lem:bound_noise_free}] \label{lem:bound_noise_free_v2}
    Consider an unknown $f: \mathcal{X} \to \mathbb{R}$ and suppose $f \in \mathcal{H}_k$. Let the \gls{gp} in \eqref{eq:gp} be trained using noise-free data $\mathcal{D}$ generated via \eqref{eq:system}, i.e., $M_i(\omega) = 0$ for all $i \in \mathbb{I}_D$. Then
    \[
        \abs{f(x) - \mu(x)} \leq \norm{f}_{\mathcal{H}_k} \sqrt{\sigma_\mathrm{GP}^2(x) - \sigma_\mathrm{WK}^2(x)}, \quad \forall \, x \in \mathcal{X}
    \]
    holds with $\mu(x)$, $\sigma_\mathrm{GP}^2(x)$, and $\sigma_\mathrm{WK}^2(x)$ from \eqref{eq:gp} and \eqref{eq:wievar}, respectively, for any $\sigma_M \ge 0$. \hfill \small{$\Box$}
\end{lemma}
\begin{proof}
    We use the shorthands $\mathbf{K}_M \doteq \mathbf{K} + \sigma_M^2 \mathbf{I}$, $\mathbf{k}_x \doteq \mathbf{k}(x)$, and $\kappa_x \doteq k(\cdot, x)$. Notice that $\mu(x) = \mathbf{y}^\top \mathbf{K}_M^{-1} \mathbf{k}_x$. We set $\mathbf{q}_x \doteq \mathbf{K}_M^{-1} \mathbf{k}_x \doteq \left[
        (q_x)_1 \, \ldots \, (q_x)_D
    \right]^\top \in \mathbb{R}^D$. Since $M_i(\omega) = 0, \forall \, i \in \mathbb{I}_D$, we have that $y_i = f(x_i), \forall \, i \in \mathbb{I}_D$. Observe that $\mu(x) = \mathbf{y}^\top \mathbf{q}_x = \sum_{i=1}^D f(x_i) (q_x)_i = \sum_{i=1}^D \inprod{f, \kappa_{x_i}}_{\mathcal{H}_k} (q_x)_i$ due to the reproducing property, cf. point ii) in Definition~\ref{def:rkhs}. Using the Cauchy-Schwarz inequality we obtain
\begin{equation*}
    \begin{split}
        \abs{f(x) - \mu(x)}^2 &= \abs{\inprod{f, \kappa_x}_{\mathcal{H}_k} - \textstyle\sum\limits_{i=1}^D \inprod{f, \kappa_{x_i}}_{\mathcal{H}_k} (q_x)_i}^2 \\
        &\leq \norm{f}_{\mathcal{H}_k}^2 \norm{\kappa_x - \textstyle\sum\limits_{i=1}^D \kappa_{x_i} (q_x)_i}^2_{\mathcal{H}_k}.
    \end{split}
\end{equation*}
Expansion of the second norm in the product above yields:
\begin{equation*}
    \begin{split}
        &\norm{\kappa_x}^2_{\mathcal{H}_k} - 2 \textstyle\sum\limits_{i=1}^D \inprod{\kappa_x, \kappa_{x_i}}_{\mathcal{H}_k} (q_x)_i \, \\&\hspace*{2cm}+ \textstyle\sum\limits_{i,j=1}^D \inprod{\kappa_{x_i}, \kappa_{x_j}}_{\mathcal{H}_k} (q_x)_i (q_x)_j \\
        &~~=k(x, x) - 2 \mathbf{k}_x^\top \mathbf{q}_x \, + \mathbf{q}_x^\top \mathbf{K} \mathbf{q}_x \\
        &~~= k(x, x) - 2 \mathbf{k}_x^\top \mathbf{K}_M^{-1} \mathbf{k}_x + \mathbf{k}_x^\top \mathbf{K}_M^{-1}  \mathbf{K} \mathbf{K}_M^{-1} \mathbf{k}_x \\
        &~\,\stackrel{\text{\eqref{eq:kervar}}}{=} \sigma_\mathrm{GP}^2(x) - \mathbf{k}_x^\top \mathbf{K}_M^{-1} \mathbf{k}_x + \mathbf{k}_x^\top \mathbf{K}_M^{-1}  \mathbf{K} \mathbf{K}_M^{-1} \mathbf{k}_x \\
        &~~= \sigma_\mathrm{GP}^2(x) - \mathbf{k}_x^\top \mathbf{K}_M^{-1} \mathbf{K}_M \mathbf{K}_M^{-1} \mathbf{k}_x + \mathbf{k}_x^\top \mathbf{K}_M^{-1}  \mathbf{K} \mathbf{K}_M^{-1} \mathbf{k}_x \\
        &~~= \sigma_\mathrm{GP}^2(x) - \mathbf{k}_x^\top \mathbf{K}_M^{-1} (\mathbf{K}_M - \mathbf{K}) \mathbf{K}_M^{-1} \mathbf{k}_x \\
        &~~= \sigma_\mathrm{GP}^2(x) - \sigma_\mathrm{WK}^2(x), 
        \end{split}
\end{equation*}
where the last step follows from $\mathbf{K}_M - \mathbf{K} = \sigma_M^2\mathbf{I}$.
Hence, we have the assertion.
\end{proof}

\begin{theorem}[Wiener kernel error bound] \label{th:bound}
     Let $f: \mathcal{X} \to \mathcal{Y}$ be the unknown function from \eqref{eq:system}, let $\gls{kermean}: \mathcal{X} \to \mathcal{Y}$ be the mean estimate \eqref{eq:kermean} based on the dataset \eqref{eq:dataset}, and let $\delta \in (0,\,1)$. Then, with a probability of at least $1-\delta$, the regression error is bounded by \eqref{eq:errorbound_def} with $\gls{bound}(x) = \gls{bound}_{\mathrm{WK}}(x)$,
     \begin{equation} \label{eq:bound_proposed}
         \gls{bound}_{\mathrm{WK}}(x) \doteq B \sqrt{\gls{kerstd}^2(x) - \gls{wiestd}^2(x)} + \beta_{\mathrm{WK}}(\delta)\gls{wiestd}(x),
     \end{equation}
     where $B \ge \norm{f}_{\mathcal{H}_k}$ and $\beta_{\mathrm{WK}}(\delta) = \sqrt{2\ln{2/\delta}}$. \hfill \small{$\Box$}
\end{theorem}
\begin{proof}
    As in the previous proof, we use the shorthand $\mathbf{K}_M \doteq \mathbf{K} + \sigma_M^2 \mathbf{I}$. Via \eqref{eq:system}, we first decompose the available output data $\mathbf{y}$ into noise-free and noise components, $\mathbf{f}$ and $\mathbf{m}$, respectively. Put differently, $\mathbf{y} = \mathbf{f} + \mathbf{m}$. Thus, we have    
    \begin{align*}
        \abs{ f(x) - \gls{kermean}(x)} &= \abs{ f(x) - \gls{gramvec}(x)^{\top}\gls{gram}_M^{-1}\mathbf{y}}\\
        & \le \abs{ f(x) - \gls{gramvec}(x)^{\top}\gls{gram}_M^{-1}\mathbf{f} } + \abs{ \gls{gramvec}(x)^{\top}\gls{gram}_M^{-1}\mathbf{m} }.
    \end{align*}
    With the \gls{rkhs} norm bound $B \ge \norm{f}_{\mathcal{H}_k}$, it holds that 
    \begin{equation} \label{eq:bound_rkhs}
        \abs{ f(x) - \gls{gramvec}(x)^{\top}\gls{gram}_M^{-1}\mathbf{f} } \le B \sqrt{\gls{kerstd}^2(x) - \gls{wiestd}^2(x)},
    \end{equation} 
    cf.~Lemma~\ref{lem:bound_noise_free_v2}. In order to obtain an upper bound on $\abs{ \gls{gramvec}(x)^{\top}\gls{gram}_M^{-1}\mathbf{m} }$, let us define the random variable $\mathbf{M} \doteq \left[M_1 \, \ldots \, M_D\right]^\top \sim \mathcal{N}\left(\mathbf{0},\, \sigma_M^2\mathbf{I}\right)$ that generated the noise realization $\mathbf{m}$ in the data via $\mathbf{m} = \mathbf{M}(\omega)$. As $\mathbf{M}$ is Gaussian-distributed, the transformed variable $\tilde{\mathbf{M}} \doteq \gls{gramvec}(x)^{\top}\gls{gram}_M^{-1}\mathbf{M}$ is also Gaussian-distributed. 
    More specifically, we have 
    \begin{align*}
        \mean{\tilde{\mathbf{M}}} &= \gls{gramvec}(x)^{\top}\gls{gram}_M^{-1} \mean{\mathbf{M}},\\
        \var{\tilde{\mathbf{M}}} &= \gls{gramvec}(x)^{\top}\gls{gram}_M^{-1} \var{\mathbf{M}} \gls{gram}_M^{-1} \gls{gramvec}(x),
    \end{align*}
    and thus $\tilde{\mathbf{M}} \sim \mathcal{N}\left(0,\,\gls{wiestd}^2(x)\right)$ with $\gls{wiestd}^2(x)$ from \eqref{eq:wievar}.
    Thus, applying the general Hoeffding inequality \cite[Thm.~2.6.3]{vershynin2018high} for (sub-) Gaussian random variables, we obtain that $\mathbb{P}\left[\lvert\tilde{\mathbf{M}}(\omega)\rvert \le c\right] \geq 1-2\exp{-c^2/(2\gls{wiestd}^2(x))}$ for some $c > 0$. Introducing the confidence parameter $\delta \in (0,\,1)$ and choosing $c \doteq \beta_{\mathrm{WK}}(\delta) \gls{wiestd}(x)$ yields that
    \begin{equation} \label{eq:bound_noise}
        \abs{ \gls{gramvec}(x)^{\top}\gls{gram}_M^{-1}\mathbf{m} } \le \beta_{\mathrm{WK}}(\delta) \gls{wiestd}(x)
    \end{equation}
    holds with a probability of at least $1-\delta$. Combining the derived bounds \eqref{eq:bound_rkhs} and \eqref{eq:bound_noise}, the probabilistic error bound \eqref{eq:errorbound_def} holds with $\gls{bound}(x)$ as in \eqref{eq:bound_proposed}.
\end{proof}
\begin{remark}[Extension to sub-Gaussian noise]
    Due to the use of the general Hoeffding inequality \cite[Thm.~2.6.3]{vershynin2018high} in the proof of Theorem~\ref{th:bound}, the proposed bound \eqref{eq:bound_proposed} trivially extends to sub-Gaussian noise by replacing $\sigma_M^2$ in \eqref{eq:wievar} with the variance proxy of the sub-Gaussian noise. Thus, the proposed bound \eqref{eq:bound_proposed} holds for the same class of noise as considered by the bounds in \cite{Chowdhury2017,abbasi2013online,Fiedler2021}.
    \hfill \small{$\Box$}
\end{remark}

In order to prove that the proposed bound \eqref{eq:bound_proposed} is tighter than the bounds presented in Table~\ref{tab:bounds}, we first show that the Wiener kernel variance $\gls{wiestd}^2$ \eqref{eq:wievar} is upper bounded by the \gls{gp} posterior variance $\gls{kerstd}^2$ \eqref{eq:kervar} for the case of Gaussian noise.
\begin{lemma}[Wiener kernel variance bound] \label{lem:bound_wieker}
    For any input domain $\mathcal{X} \neq \emptyset$ and any kernel function $k: \mathcal{X} \times \mathcal{X} \to \mathbb{R}$ it holds that $\gls{wiestd}^2(x) \le \gls{kerstd}^2(x)$, $\forall \, x \in \mathcal{X}$. \hfill \small{$\Box$}
\end{lemma}
\begin{proof}
    By contradiction. Assume there exists $x \in \mathcal{X}$ such that $\gls{wiestd}^2(x) > \gls{kerstd}^2(x)$. Pick any $f \in \mathcal{H}_k \setminus \{0\}$. Then, by the statement of Lemma~\ref{lem:bound_noise_free_v2}, the norm $\norm{f}_{\mathcal{H}_k}$ does not attain its value in $[0, \infty)$. Thus, $f \notin \mathcal{H}_k$. The degenerate case $\mathcal{H}_k = \{0\}$ implies that $k(x,x^\prime) = 0, \forall \, x, x^\prime \in \mathcal{X}$. This, however, means that $\gls{wiestd}^2(x) = \gls{kerstd}^2(x) = 0, \forall \, x \in \mathcal{X}$. 
\end{proof}
For the alternative proof of Lemma~\ref{lem:bound_wieker} in the feature space $\mathbb{R}^{n_\phi}$, we refer to Appendix~\hyperref[app:wieker_bound_v2]{A}.

The following result shows that, under mild conditions, the proposed error bound \eqref{eq:bound_proposed} is tighter than the bounds presented in Table~\ref{tab:bounds}.
\begin{lemma}[Tightness of Wiener error bound] \label{lem:bound_tight}
    Consider the bound $\gls{bound}_{\mathrm{WK}}$ from Theorem~\ref{th:bound} and the bounds $\gls{bound}_1$, $\gls{bound}_2$ from Table~\ref{tab:bounds}. For all $x \in \mathcal{X}$ and $\delta \in (0,\,1)$, the following holds:
    \begin{itemize}
        \item[a)] $\gls{bound}_{\mathrm{WK}}(x) < \gls{bound}_1(x)$ if $\gamma(\gls{gram}) \doteq \det{\sigma_M^{-2}\gls{gram} + \mathbf{I}} > 4$,
        \item[b)] $\gls{bound}_{\mathrm{WK}}(x) < \gls{bound}_2(x)$ if $D \ge 2$. 
        \hfill \small{$\Box$}
    \end{itemize} 
\end{lemma}
\begin{proof}
    Since $B \gls{kerstd}(x) \ge B \sqrt{\gls{kerstd}^2(x) - \gls{wiestd}^2(x)}$ holds, it is sufficient to compare the remaining terms. First, note that the parameter $\beta_{\mathrm{WK}}$ from Theorem~\ref{th:bound} can be expressed as $\beta_{\mathrm{WK}}(\delta) = \sqrt{\ln{4} + 2\ln{1/\delta}}$.
       
    a) $\gls{bound}_{\mathrm{WK}}(x) < \gls{bound}_1(x)$: For $\gamma(\gls{gram}) > 4$, we have $\ln{\gamma(\gls{gram})} > \ln{4}$ and thus $\beta_1(\delta) > \beta_{\mathrm{WK}}(\delta)$ for all $\delta \in (0,\,1)$. As $\gls{wiestd}^2(x) \le \gls{kerstd}^2(x)$ holds via Lemma~\ref{lem:bound_wieker}, the assertion follows.        
    
    b) $\gls{bound}_{\mathrm{WK}}(x) < \gls{bound}_2(x)$: For $D \ge 2$, it holds that $\beta_2(\delta) > \sqrt{2 + 2\ln{1/\delta}}$. Since $2 > \ln{4}$, we have $\beta_2(\delta) > \beta_{\mathrm{WK}}(\delta)$ for $\delta \in (0,\,1)$. Moreover, we can write $\gls{wiestd}^2(x)$ from \eqref{eq:wievar} as
    \begin{equation*} 
        \gls{wiestd}^2(x) = \sigma_M^2  \norm{ \left(\gls{gram} + \sigma_M^2 \mathbf{I}\right)^{-1}\gls{gramvec}(x)}^2,
    \end{equation*}
    from which we obtain the correspondence
    \begin{equation} \label{eq:wievar_connection}
        \sigma_M  \norm{ \left(\gls{gram} + \sigma_M^2 \mathbf{I}\right)^{-1}\gls{gramvec}(x)} = \gls{wiestd}(x).
    \end{equation}
    Thus, the assertion follows.

\end{proof}
\begin{remark}[Satisfying the conditions of Lemma~\ref{lem:bound_tight}]~\\
    In practice, the condition $\gamma(\gls{gram}) > 4$ for Lemma~\ref{lem:bound_tight}.b) is easily satisfied  for large enough $D$: Denote the eigenvalues of $\gls{gram}$ as $\lambda_1,\,\dots,\,\lambda_D$. As $\gls{gram}$ is positive semidefinite, $\lambda_i \ge 0$ for $i \in \mathbb{I}_D$. Thus, we can write $\gamma(\gls{gram}) = \prod_{i=1}^{D}(\sigma_M^{-2}\lambda_i + 1)$, which is a monotonically increasing function in $D$ as $\sigma_M^{-2}\lambda_i \ge 0$, with a minimum value of $\sigma_M^{-2} \min_{x\in\mathcal{X}} k(x,\,x) + 1$. The increase of $\gamma(\gls{gram})$ with respect to $D$ is related to the information gain: $\gamma(\gls{gram})$ increases if the newly added data point yields a nonzero eigenvalue (i.e., the new data point is not redundant to previous data points), which is practically the case when exploring the hypothesis space to reduce uncertainty, e.g., in the context of \gls{bo}.

    Lemma~\ref{lem:bound_tight}.c) also holds for $D = 1$ if one restricts the domain of the confidence parameter $\delta$ to the practically reasonable range $\delta \in (0,\,0.5)$, since then $\beta_2(\delta) > \sqrt{1 + 2\sqrt{\ln{2}} + 2\ln{1/\delta}} > \beta_{\mathrm{WK}}(\delta)$. Interestingly, the Wiener kernel variance $\gls{wiestd}^2(x)$ from \eqref{eq:wievar} is present in the bound $\gls{bound}_2(x)$ from Table~\ref{tab:bounds} \cite[Prop.~2]{Fiedler2021} due to the correspondence \eqref{eq:wievar_connection}, but has not been specified as such.
    \hfill \small{$\Box$}
\end{remark}

Using Lemma~\ref{lem:bound_tight}, we can show that the usage of the proposed confidence bound \eqref{eq:bound_proposed} in Theorem~\ref{th:bound} allows for enlarged safe regions in \gls{bo}.
\begin{lemma}[Enlarged safe region] \label{lem:tightsafety}
    Denote as $\alpha^f_{\mathrm{WK}}$ the UCB \eqref{eq:ucb} of the safety constraint function $f$ using the error bound $\gls{bound}_{\mathrm{WK}}$ from Theorem~\ref{th:bound}. 
    Let $\alpha^f_{1}$, $\alpha^f_{2}$ denote the UCBs \eqref{eq:ucb} using the error bounds $\gls{bound}_1$, $\gls{bound}_2$ from Table~\ref{tab:bounds}. Under the conditions in Lemma~\ref{lem:bound_tight}, the safe region $\glsd{xsafe}(\alpha^f_{\mathrm{WK}})$ satisfies $\glsd{xsafe}(\alpha^f_i) \subset \glsd{xsafe}(\alpha^f_{\mathrm{WK}})~\forall x \in \mathcal{X}$, $\delta \in (0,\,1)$, and $i \in \mathbb{I}_2$.  \hfill \small{$\Box$}
\end{lemma}
\begin{proof}
    The region $\glsd{xsafe}(\alpha^f)$ is defined via the constraint $\alpha^f(x) \le 0$. Therefore, $\glsd{xsafe}(\alpha^f_i) \subset \glsd{xsafe}(\alpha^f_{\mathrm{WK}})$ is satisfied if $\alpha^f_{\mathrm{WK}}(x) < \alpha^f_i(x)$ holds $\forall x \in \mathcal{X}$. From \eqref{eq:ucb}, we have 
    \begin{align*}
        \alpha^f_{\mathrm{WK}}(x) < \alpha^f_i(x) & \Leftrightarrow \gls{kermean}(x) + \gls{bound}_{\mathrm{WK}}(x) < \gls{kermean}(x) + \gls{bound}_i(x), \\
        & \Leftrightarrow \gls{bound}_{\mathrm{WK}}(x) < \gls{bound}_i(x),
    \end{align*}
    which holds for all $x \in \mathcal{X}$, $\delta \in (0,\,1)$, and $i \in \mathbb{I}_2$ via Lemma~\ref{lem:bound_tight}. This concludes the proof.
\end{proof}

Theorem~\ref{th:bound} and Lemma~\ref{lem:bound_tight} show that the separation of aleatoric and epistemic uncertainty via Wiener kernel regression allows for the derivation of probabilistic error bounds that are tighter than common bounds from related works. As the results of this section are of a qualitative nature, we evaluate the numerical improvement of using the proposed bound \eqref{eq:bound_proposed} over bounds in Table~\ref{tab:bounds} in the next section.

%% file: 5_eval.tex
\section{NUMERICAL EXAMPLE} \label{sec:eval}
We evaluate the proposed error bound from Theorem~\ref{th:bound} in the context of safe \gls{bo}. The source code is available online at \url{https://github.com/OptCon/SafeBO_WKR}.

In~\cite{faulwasser2024,Beckers2016} the scalar system 
\begin{equation*}
    y_t = g(x_t) + M_t(\omega) \doteq 0.01 x_t^3 - 0.2 x_t^2 + 0.2 x_t + M_t(\omega),
\end{equation*}
where $M_t \sim \mathcal{N}\left(0,\,\sigma_W^2\right)$ is additive i.i.d. noise with $\sigma_W = 1$ is considered. We have the safety constraint $f(x_t) \doteq -g(x_t) + g_{\min} \le 0$ with $g_{\min} \doteq -5.168$ and the input domain is $\mathcal{X} \doteq \{x \in \mathbb{R} \mid \abs{x} \le 5\}$. The safe action $\gls{xsafe} = 5$ is known. For \gls{gp} regression, we use a squared-exponential kernel $k(x,\,x') \doteq \sigma_{\mathrm{SE}}^2 \exp{-\norm{x-x'}^2/\big(2l_{\mathrm{SE}}^2\big)}$ with $\sigma_{\mathrm{SE}} = 4.21$ and $l_{\mathrm{SE}}=3.59$ as in \cite{Beckers2016}. Note that these hyperparameters are fixed and are not updated throughout Algorithm~\ref{alg:safeBO}.

The goal is to find the maximum $g_{\mathrm{opt}} = 0.0513$ of the unknown function $g$ via \gls{bo}: at the learning step~$t$, a new action $x_t$ is chosen via Algorithm~\ref{alg:safeBO}. For the safe \gls{bo} problem~\eqref{eq:safeBO}, we use the acquisition function \eqref{eq:ucb} for different confidence bounds $\gls{bound}(x_t)$ detailed below. The constraint in \eqref{eq:safeBO} guarantees that the safety constraint $g(x_t) \ge g_{\min}$ is satisfied with high confidence $1-\delta$ (see Lemma~\ref{lem:safety}), where $\delta = 0.001$.

In a Monte Carlo simulation of $100$ runs, the safe \gls{bo} algorithm is applied for $t\in\mathbb{I}_{100}$ subject to randomly drawn noise realizations $M_t(\omega)$. With the discussion from Remark~\ref{rem:bounds} in mind, we compare the performance of the safe \gls{bo} algorithm under the usage of the bounds from Table~\ref{tab:bounds}:
\begin{itemize}
    \item Using $\gls{bound}_{\mathrm{WK}}(x)$ according to Theorem~\ref{th:bound}.
    \item Using $\gls{bound}_1(x)$ from
    \cite[Thm.~3.11]{abbasi2013online}.
    \item Using  $\gls{bound}_2(x)$ from
    \cite[Prop.~2]{Fiedler2021}.
\end{itemize}
We use the \gls{rkhs} norm bound $B = 3$ for all methods. This estimate is found by fitting a function $\sum_{i=1}^n \alpha_i k(\cdot, \tilde{x}_i)$ to the true function $g$ using a selection of points $\lbrace \tilde{x}_i \rbrace_{i=1}^n$ with Gram matrix $\tilde{\mathbf{K}}$. If $g$ is approximated sufficiently well with $\alpha^\star = [\alpha^\star_i]_{i=1}^n$, then the result is set to $B \doteq\sqrt{\alpha^{\star\top} \tilde{\mathbf{K}} \alpha^\star}$.  

The top of Figure~\ref{fig:regret} shows the cumulative regret
\begin{equation}
    R_t \doteq \textstyle\sum_{i=1}^t \left(g_{\mathrm{opt}} - g(x_t)\right)
\end{equation}
of the different methods over the learning steps, as well as the $75\%$ confidence intervals over the Monte Carlo runs.
When compared to the proposed method using $\gls{bound}_{\mathrm{WK}}$, the comparison methods using $\gls{bound}_{1}$ and $\gls{bound}_{2}$ emit a mean increase in cumulative regret at end time $t=100$ of $116.92\%$ and $210.22\%$, respectively. The bottom of Figure~\ref{fig:regret} illustrates the evolution of the safe region $\glsd{xsafe}$ for the different methods over the learning steps. It can be seen that the proposed method yields the fastest increase of size of $\glsd{xsafe}$, converging after $20$ learning steps. In contrast, the comparison methods not only take significantly longer to converge, but also end in a smaller safe region than the proposed method (cf. Lemma~\ref{lem:tightsafety}). 

\begin{figure}[!t]
    \centering
    \includegraphics[page = 2, clip, trim=2cm 16.7cm 10cm 5.2cm]{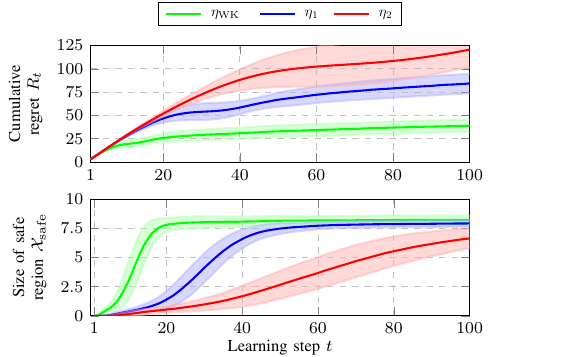}
    \caption{Evolution of the cumulative regret and the size of the safe region $\glsd{xsafe}$ over learning steps for all comparison methods. Shaded areas show the $75\%$-confidence intervals over all runs. Thick lines represent the mean.}
    \label{fig:regret}
\end{figure}

Figure~\ref{fig:beta} illustrates the evolution of the bound parameters $\beta_{\mathrm{WK}}$, $\beta_1$, and $\beta_2$ (see Table~\ref{tab:bounds}) over the learning steps. For the considered example, the conditions detailed in Lemma~\ref{lem:bound_tight} are satisfied for all $t \ge 1$ and, thus, the proposed bound $\gls{bound}_{\mathrm{WK}}$ is the tightest. While Lemma~\ref{lem:bound_tight} is only of qualitative nature, the quantitative benefit of using the proposed bound becomes evident: In comparison to the bound parameters $\beta_1$, $\beta_2$ from the comparison methods, the proposed bound parameter $\beta_{\mathrm{WK}}$ is independent of the data size and the chosen actions $x_t$, and thus constant over the learning steps. The parameter $\beta_1$ depends on the actually obtained data via the Gram matrix $\gls{gram}$ and thus varies in each run. The parameter $\beta_2$ grows with $\sqrt{t}$ and therefore faster than the other bound parameters. That is the reason why the method based on $\gls{bound}_{2}$ performs worse than the method based on $\gls{bound}_{1}$, despite the use of the Wiener kernel variance \eqref{eq:wievar} in the bound $\gls{bound}_2$, see \eqref{eq:wievar_connection}.

\begin{figure}[!t]
    \centering
    \includegraphics[page = 2, clip, trim=2cm 5.6cm 10cm 18.45cm]{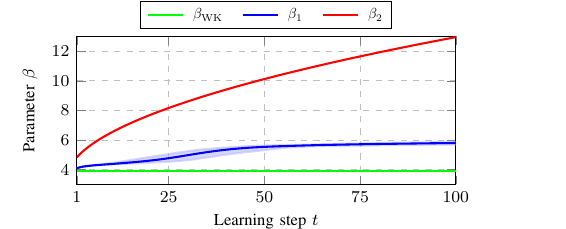}
    \caption{Evolution of the bound parameters $\beta_{\mathrm{WK}}(\delta)$ (see Theorem~\ref{th:bound}) and $\beta_1(\delta)$, $\beta_2(\delta)$ (see Table~\ref{tab:bounds}) over learning steps $t$ for $\delta = 0.001$. The shaded area shows the maximum and minimum of $\beta_1(\delta)$ over all runs.}
    \label{fig:beta}
\end{figure}

\begin{figure*}[!ht]
    \centering
    \includegraphics[page = 1, clip, trim=2cm 10.4cm 2cm 10.4cm]{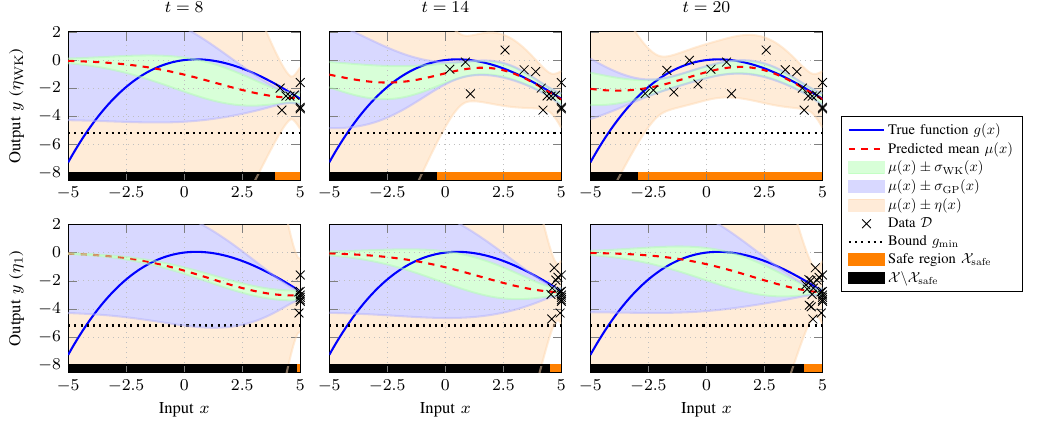}
    \caption{Comparison of the learning progress based on the proposed error bound $\gls{bound}_{\mathrm{WK}}$ from Theorem~\ref{th:bound} (upper row) and based on the error bound $\gls{bound}_{1}$ from Table~\ref{tab:bounds} \cite[Th.~3.11]{abbasi2013online} (lower row) for learning steps $t=8$ (left column), $t=14$ (middle column), and $t=20$ (right column).}
    \label{fig:learning}
\end{figure*}

Finally, Figure~\ref{fig:learning} illustrates the learning progress at learning steps $t=8$, $t=14$, and $t=20$ of the proposed method using $\gls{bound}_{\mathrm{WK}}$ compared to the comparison method based on $\gls{bound}_{1}$ in more detail for one exemplary run. The proposed method quickly manages to extend the safe region due to the tighter confidence bound from Theorem~\ref{th:bound}, while the more conservative bound $\gls{bound}_{1}$ of the comparison method results in more conservative learning progress. Figure~\ref{fig:learning} also shows the confidence tubes spanned by the \gls{gp} posterior variance $\gls{kerstd}^2$ \eqref{eq:kervar} and the Wiener kernel variance $\gls{wiestd}^2$ \eqref{eq:wievar}. The standard deviation $\gls{wiestd}$ is not only smaller than $\gls{kerstd}$ (Lemma~\ref{lem:bound_wieker}), but especially smaller at domains that are far away from observed data points. This indicates that the proposed bound~\eqref{eq:bound_proposed} is especially tight at yet unexplored regions in the hypothesis space.

%% file: 6_conclusion.tex
\section{CONCLUSIONS} \label{sec:conclusion}
In this work, we presented a novel probabilistic error bound for \gls{gp} regression subject to Gaussian measurement noise by leveraging Wiener kernel regression. Under mild assumptions on the collected data, the proposed bound is shown to be tighter than commonly used bounds documented in the literature. As a consequence, using the proposed bound in safe \gls{bo} leads to enlarged safe regions and thus improved performance of the optimization scheme. The potential of safe \gls{bo} utilizing the proposed bound has been demonstrated in simulation, yielding numerical evidence for the favorable tightness of the proposed bound compared to related approaches.

While we have focused on Gaussian noise in this work, Wiener kernel regression enables the derivation of probabilistic error bounds for the more general case of non-Gaussian noise from arbitrary ${L}^2$ probability spaces. A thorough investigation of this general case and a regret analysis is subject of future work.

%% file: appendix.tex
\section*{APPENDIX}
\subsection{Proof of Lemma~\ref{lem:bound_wieker} in the feature space}
\label{app:wieker_bound_v2}
\begin{proof}
    The difference between \eqref{eq:kervar} and \eqref{eq:wievar} yields
    \[
    \begin{split}
        &\gls{kerstd}^2(x) - \gls{wiestd}^2(x) \\
        &~= k(x,x) - \mathbf{k}(x)^\top (\sigma_M^2 \mathbf{I} + \gls{gram})^{-1} \mathbf{k}(x) \\
        &~~~~ - \sigma_M^2 \mathbf{k}(x)^\top (\sigma_M^2 \mathbf{I} + \gls{gram})^{-2} \mathbf{k}(x) \\
        &~= \phi(x)^\top \phi(x) - \phi(x)^\top \Phi (\sigma_M^2 \mathbf{I} + \Phi^\top \Phi)^{-1} \Phi^\top \phi(x) \\
        &~~~~ - \sigma_M^2 \phi(x)^\top \Phi (\sigma_M^2 \mathbf{I} + \Phi^\top \Phi)^{-2} \Phi^\top \phi(x) = \phi(x)^\top A \phi(x),
    \end{split}
    \]
    whereby $\Phi \doteq \left[\phi(x_1) \, \ldots \, \phi(x_D)\right] \in \mathbb{R}^{n_\phi \times D}$ and $A$ reads
    \begin{equation*}
        A \doteq \mathbf{I} - \Phi (\sigma_M^2 \mathbf{I} + \Phi^\top \Phi)^{-1} \Phi^\top - \sigma_M^2 \Phi (\sigma_M^2 \mathbf{I} + \Phi^\top \Phi)^{-2} \Phi^\top.
    \end{equation*}
    We prove the assertion by certifying that $A$ is positive semi-definite. Recall that as a consequence of the Woodbury matrix identity, for any matrices $U$, $V$, and $P$ of suitable dimensions, as well as for any scalar $a > 0$, the following identities hold:
        i)   $a (a \mathbf{I} + P)^{-1} = \mathbf{I} - P (a \mathbf{I} + P)^{-1}$,
        ii)  $a (a \mathbf{I} + UV)^{-1} = \mathbf{I} - U (a \mathbf{I} + VU)^{-1} V$, and
        iii) $(a \mathbf{I} + UV)^{-1} U = U (a \mathbf{I} + VU)^{-1}$.
    By taking $U = \Phi$, $V = \Phi^\top$, $a = \sigma_M^2$, and $P = UV = \Phi \Phi^\top$, the matrix $A$ can be simplified as
    \begin{align*}
        A &= \sigma_M^2 (\sigma_M^2 \mathbf{I} + \Phi \Phi^\top)^{-1} \\
        &~~~ - \sigma_M^2 (\sigma_M^2 \mathbf{I} + \Phi \Phi^\top)^{-1} \Phi \Phi^\top (\sigma_M^2 \mathbf{I} + \Phi \Phi^\top)^{-1} \\
        &= \sigma_M^2 (\sigma_M^2 \mathbf{I} + P)^{-1} \left( \mathbf{I} - P (\sigma_M^2 \mathbf{I} + P)^{-1} \right) \\
        &= \sigma_M^4 (\sigma_M^2 \mathbf{I} + P)^{-2} = \sigma_M^4 (\sigma_M^2 \mathbf{I} + \Phi \Phi^\top)^{-2}.
    \end{align*}
    Thus, $A$ is positive-definite, which completes the proof.
\end{proof}

\begin{remark}[Strictness of Wiener kernel variance bound]
    Since the matrix $A$ in the above proof is positive-definite, the inequality $\gls{wiestd}^2(x) \le \gls{kerstd}^2(x)$, $\forall \, x \in \mathcal{X}$ is strict, except for the points $\{x \in \mathcal{X} \mid \phi(x) = 0\}$, or, 
    equivalently, $\{x \in \mathcal{X} \mid k(x,x) = 0\}$. \hfill \small{$\Box$}
\end{remark}